\documentclass[12pt,reqno]{amsart}
\usepackage[margin=1.25in]{geometry}
\usepackage[foot]{amsaddr}
\usepackage{microtype}
\usepackage{amssymb}
\usepackage{amsfonts}
\usepackage{hyperref}
\usepackage{algorithm}
\usepackage{algorithmicx}
\usepackage{algpseudocode}

\usepackage{enumitem}

\renewcommand{\geq}{\geqslant}
\renewcommand{\leq}{\leqslant}
\newcommand{\eps}{\varepsilon}
\theoremstyle{plain}

\newtheorem{lemma}{Lemma}

\newtheorem{prop}{Proposition}
\theoremstyle{definition}
\newtheorem{definition}{Definition}
\theoremstyle{remark}
\newtheorem{example}{Example}
\newtheorem{remark}{Remark}
\usepackage{graphicx}
\usepackage[normalem]{ulem}
\DeclareMathOperator{\Pois}{Pois}
\DeclareMathOperator{\Hom}{Hom}
\DeclareMathOperator{\diag}{diag}
\DeclareMathOperator{\Orth}{O}
\DeclareMathOperator{\GL}{GL}
\DeclareMathOperator{\divergence}{div}
\DeclareMathOperator*{\argmin}{argmin}
\DeclareMathOperator{\prox}{prox}

\newcommand{\R}{\mathbf R}

\newcommand{\Z}{\mathbf Z}
\newcommand{\N}{\mathbf N}

\newcommand{\SO}{\mathrm{SO}}

\newcommand{\SE}{\mathrm{SE}}

\DeclareMathOperator{\TV}{TV}
\DeclareMathOperator{\id}{id}

\begin{document}
\title{Equivariant neural networks for inverse problems}
\author{Elena~Celledoni$^1$, Matthias~J.~Ehrhardt$^2$, Christian~Etmann$^3$, Brynjulf~Owren$^1$, Carola-Bibiane~Sch\"onlieb$^3$ and Ferdia~Sherry$^{3\dagger}$}

\address{$^1$Department of Mathematical Sciences, NTNU, N-7491 Trondheim, Norway}
\address{$^2$Institute for Mathematical Innovation, University of Bath, Bath BA2 7JU, UK}
\address{$^3$Department of Applied Mathematics and Theoretical Physics, University of Cambridge, Wilberforce Road, Cambridge CB3 0WA, UK}
\address{$^\dagger$Corresponding author, email address: \href{mailto:fs436@cam.ac.uk}{\texttt{\textup{fs436@cam.ac.uk}}}}
\date{\today}

\begin{abstract}
In recent years the use of convolutional layers to encode an inductive bias (translational equivariance) in neural networks has proven to be a very fruitful idea. The successes of this approach have motivated a line of research into incorporating other symmetries into deep learning methods, in the form of group equivariant convolutional neural networks. Much of this work has been focused on roto-translational symmetry of $\R^d$, but other examples are the scaling symmetry of $\R^d$ and rotational symmetry of the sphere. In this work, we demonstrate that group equivariant convolutional operations can naturally be incorporated into learned reconstruction methods for inverse problems that are motivated by the variational regularisation approach. Indeed, if the regularisation functional is invariant under a group symmetry, the corresponding proximal operator will satisfy an equivariance property with respect to the same group symmetry. As a result of this observation, we design learned iterative methods in which the proximal operators are modelled as group equivariant convolutional neural networks. We use roto-translationally equivariant operations in the proposed methodology and apply it to the problems of low-dose computerised tomography reconstruction and subsampled magnetic resonance imaging reconstruction. The proposed methodology is demonstrated to improve the reconstruction quality of a learned reconstruction method with a little extra computational cost at training time but without any extra cost at test time.
\end{abstract}

\maketitle
\section{Introduction}

Deep learning has recently had a large impact on a wide variety of fields; research laboratories have published state-of-the-art results applying deep learning to sundry tasks such as playing Go~\cite{silver_mastering_2016}, predicting protein structures~\cite{senior_improved_2020} and generating natural language~\cite{brown_language_2020}. In particular, deep learning methods have also been developed to solve inverse problems, with some examples being~\cite{jin_deep_2017,adler_solving_2017,lunz_adversarial_2018}. In this work we investigate the use of equivariant neural networks for solving inverse imaging problems, i.e.~inverse problems where the solution is an image. Convolutional neural networks (CNNs)~\cite{lecun_convolutional_1998} are a standard tool in deep learning methods for images. By learning convolutional filters, CNNs naturally encode translational symmetries of images: if $\tau_h$ is a translation by $h\in\R^d$, and $k, f$ are functions on $\R^d$, we formally have the following relation (translational equivariance) 
\begin{equation}
  \tau_h [k\ast f] = k \ast [\tau_h f].
    \label{eq:translational_equivariance}
\end{equation}
This allows learned feature detectors to detect features regardless of their position (though not their orientation or scale) in an image. In many cases it may be desirable for these learned feature detectors to also work when images are transformed under other group transformations, i.e.\ one may ask that a property such as Equation~\eqref{eq:translational_equivariance} holds for a more general group transformation than the group of translations $\{\tau_h|h\in\R^d\}$. If natural symmetries of the problem are not built into the machine learning method and are not present in the training data, in the worst case, it can result in catastrophic failure as illustrated in Figure~\ref{fig:intro_catastrophic_failure}.
\begin{figure}[!htb]
  \centering
  \includegraphics[scale=1]{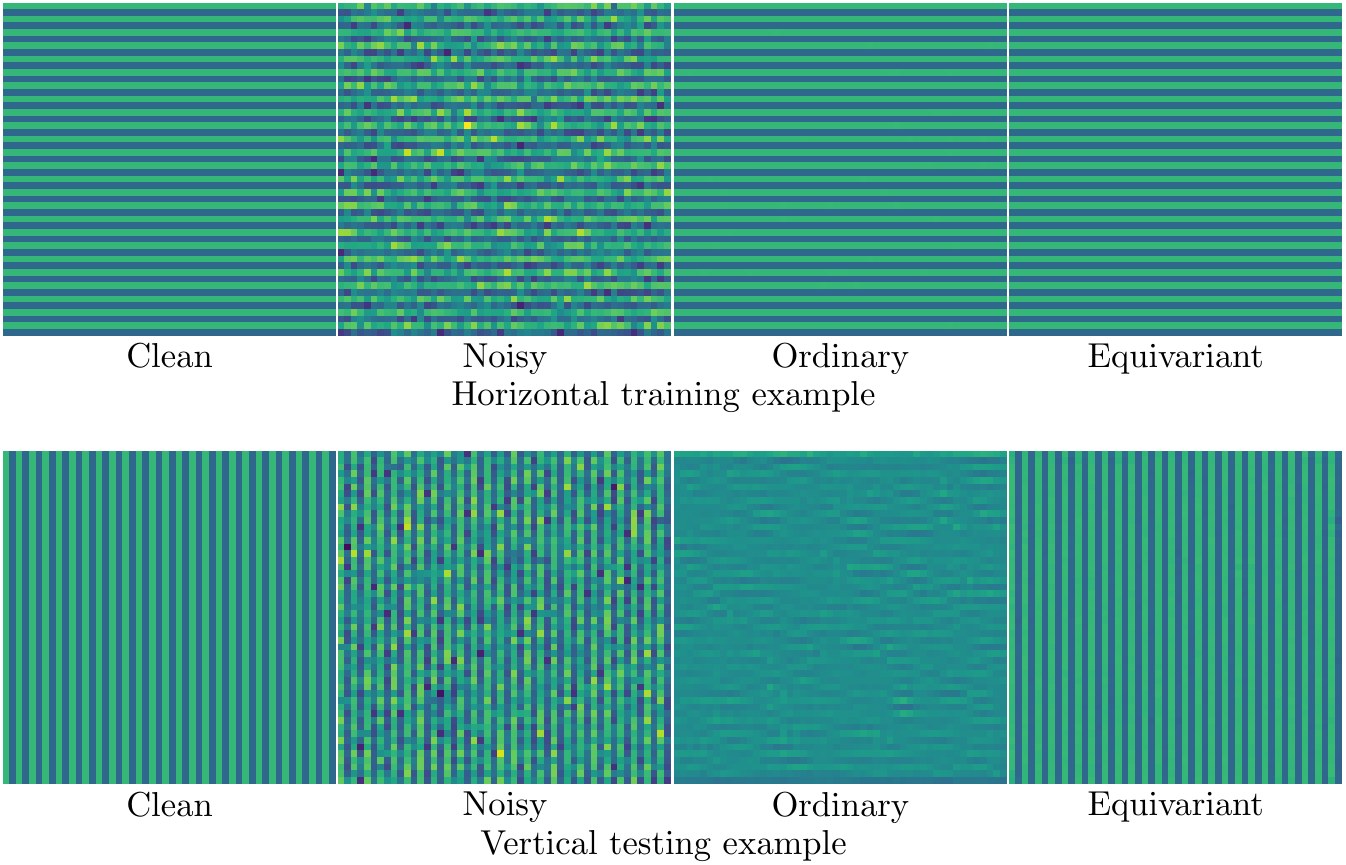}
    \caption{Roto-translationally (`Equivariant') and just translationally (`Ordinary') equivariant filters are trained to denoise on a single pair of ground truth and noisy images (`Clean' and `Noisy' in the top row), giving perfect denoising results on the training example. In the bottom row, we see the result of testing the learned filters on a rotated version of the training image; the ordinary filter completely fails at recovering the ground truth, whereas the equivariant filter performs as well as it did on the training image.}
\label{fig:intro_catastrophic_failure}
\end{figure}

To some extent, this problem is circumvented by augmenting the training data through suitable transformations, but it has been shown in classification and segmentation tasks that it is still beneficial to incorporate known symmetries directly into the architecture used, especially if the amount of training data is small~\cite{bekkers_roto-translation_2018, worrall_harmonic_2017, weiler_general_2019}. Furthermore, training on augmented data is not enough to guarantee that the final model satisfies the desired symmetries. There has recently been a considerable amount of work in this direction, in the form of group equivariant CNNs. Most of the focus has been on roto-translational symmetries of images~\cite{cohen_group_2016,dieleman_exploiting_2016,bekkers_roto-translation_2018, weiler_general_2019}, though there is also some work on incorporating scaling symmetries~\cite{sosnovik_scale-equivariant_2019,worrall_deep_2019} and even on equivariance to arbitrary Lie group symmetries~\cite{finzi_generalizing_2020}.

As mentioned before, we will concern ourselves with solving inverse imaging problems: given measurements $y$ that are related to an underlying ground truth image $u$ through a model
\begin{equation}
  \label{eq:inverse_problem_intro}
  y = \mathfrak{N}(A(u)),
\end{equation}
with $A$ the so-called forward operator and $\mathfrak{N}$ a noise-generating process, the goal is to estimate the image $u$ from the measurements $y$ as well as possible. Typical examples of inverse imaging problems include the problem of recovering an image from its line integrals as in computerised tomography~(CT)~\cite{hounsfield_computerized_1973}, or recovering an image from subsampled Fourier measurements as in magnetic resonance imaging~(MRI)~\cite{lauterbur_image_1973,mansfield_diffraction_1975}. The solution of an inverse problem is often complicated by the presence of ill-posedness: a problem is said to be well-posed in the sense of Hadamard~\cite{hadamard_sur_1902} if it satisfies a set of three conditions (existence of a solution, its uniqueness, and its continuous dependence on the measurements), and ill-posed if any of these conditions fail. 

It is a natural idea to try to apply equivariant neural networks to solve inverse imaging problems: there is useful knowledge about the relationship between a ground truth image and its measurements in the form of $A$ and the symmetries in both the measurement and image domain (the range and domain of $A$ respectively). Furthermore, training data tends to be considerably less abundant in medical and scientific imaging than in the computer vision and image analysis tasks that are typical of the deep learning revolution, such as ImageNet classification~\cite{krizhevsky_imagenet_2012}. This suggests that the lower sample complexity of equivariant neural networks (as compared to ordinary CNNs) may be harnessed in this setting with scarce data to learn better reconstruction methods. Finally, end users of the methods, e.g.~medical practitioners, are often skeptical of `black-box' methods and guarantees on the behaviour of the method, such as equivariance of the method to certain natural image transformations, may alleviate some of the concerns that they have.

We investigate the use of equivariant neural networks within the framework of learned iterative reconstruction methods~\cite{adler_solving_2017,putzky_recurrent_2017}, which constitute some of the most prototypical deep learning solutions to inverse problems. The designs of these methods are motivated by classical variational regularisation approaches~\cite{engl_regularization_1996}, which propose to overcome the ill-posedness of an inverse problem by estimating its solution as
\begin{equation}
  \label{eq:var_reg_intro}
  \hat u = \argmin_u d(A(u), y) + J(u),
\end{equation}
with $d$ a measure of discrepancy motivated by our knowledge of the noise-generating process $\mathfrak N$ and $J$ is a regularisation functional incorporating prior knowledge of the true solution. Learned iterative reconstruction methods, also known as unrolled iterative methods, are designed by starting from a problem such as Problem~\eqref{eq:var_reg_intro}, choosing an iterative optimisation method to solve it, truncating that method to a finite number of iterations, and finally replacing parts of it (e.g.~the proximal operators) by neural networks. We will show that these neural networks can naturally be chosen to be equivariant neural networks, and that doing so gives improved performance over choosing them to be ordinary CNNs. More precisely, our contributions in this work are as follows:

\subsection*{Our contributions}
We show that invariance of a functional to a group symmetry implies that its proximal operator satisfies an equivariance property with respect to that group. This insight can be combined with the unrolled iterative method approach: it makes sense for a regularisation functional to be invariant to roto-translations if there is no prior knowledge on the orientation and position of structures in the images, in which case the corresponding proximal operators are roto-translationally equivariant.

Motivated by these observations, we build learned iterative methods using roto-translationally equivariant building blocks. We show in a supervised learning setting that these methods outperform comparable methods that only use ordinary convolutions as building blocks, when applied to a low-dose CT reconstruction problem and a subsampled MRI reconstruction problem. This outperformance is manifested in two main ways: the equivariant method is better able to take advantage of small training sets than the ordinary one, and its performance is more robust to transformations that leave images in orientations not seen during training.


\section{Notation and background on groups and representations}
\label{sec:background}
In this section, we give an overview of the main concepts regarding groups and representations that are required to follow the main text. By a group $G$, we mean a set equipped with an associative binary operation $\cdot : G\times G\to G$ (usually the dot is omitted in writing), furthermore containing a neutral element $e$, such that $e\cdot g = g\cdot e = g$ for all $g\in G$ and a unique inverse $g^{-1}$ for each group element $g$, such that $g\cdot g^{-1} = g^{-1}\cdot g = e$. Given groups $G$ and $H$, we say that a map $\phi: G\to H$ is a group homomorphism if it respects the group structures:
\[\phi(g_1g_2) = \phi(g_1)\phi(g_2)\quad\text{for any}\quad g_1, g_2\in G.\]
Groups can be naturally used to describe symmetries of mathematical objects through the concept of group actions. Given a group $G$ and set $X$, we say that $G$ acts on $X$ if there is a function $T:G\times X \to X$ (the application of which we stylise as $T_g[x]$ for $g\in G, x\in X$) that obeys the group structure in the sense that
\begin{equation}
  \label{eq:group_action}
  T_{g_1}\circ T_{g_2} = T_{g_1 g_2}\quad\text{for any}\quad g_1,g_2\in G
\end{equation}
and $T_e = \id$. That is, the group action can be thought of as a group homomorphism from $G$ to the permutation group of $X$. If there is no ambiguity, the group action may just be written as $T_g[x]= g\cdot x = gx$. An important type of group actions is given by the group representations. If $V$ is a vector space, we will denote by $\GL(V)$ its general linear group, the group of invertible linear maps $V\to V$, with the group operation given by composition. A representation $\rho:G\to\GL(V)$ of a group $G$ which acts on $V$ is a group homomorphism, and so corresponds to a linear group action $T$ of $G$ on $V$: $\rho(g)x =T_g[x]$ for $x\in V$ and $g\in G$. Given a vector space $V$, any group $G$ has a representation on $V$ given by $\rho(g)=I$, which is the so-called trivial representation. If $V$ is additionally a Hilbert space, we will call $\rho$ a unitary representation if $\rho(g)$ is a unitary operator for each $g\in G$, i.e.~$\|\rho(g)x\|=\|x\|$ for all $x\in V$. Given a finite group $G=\{g_1,\ldots, g_n\}$, we can define the so-called regular representation $\rho$ of $G$ on $\R^n$ by
\[\rho(g_i)e_j = e_k,\]
where $\{e_1,\ldots,e_n\}$ is a basis of $\R^n$ and $k$ is such that $g_i g_j =g_k$. With this representation, each $\rho(g)$ is a permutation matrix, so $\rho$ is a unitary representation if the basis $\{e_1,\ldots,e_n\}$ is orthonormal.

In this work, the groups that we will consider take the form of a group of isometries on $\R^d$. These groups are represented by a semi-direct product $G=\R^d \rtimes H$, where $H$ is a subgroup of the orthogonal group $\Orth(d)$ of rotations and reflections:
\[\Orth(d) = \{R\in \GL(\R^d) | R^T = R^{-1}\}.\]
An important subgroup of $\Orth(d)$ is the special orthogonal group $\SO(d) = \{A\in \Orth(d)|\det(A)=1\}$, which represents the set of pure rotations in $\Orth(d)$.
Each element of the semi-direct product $G$ can be identified with a unique pair $(t, R)$ of $t\in \R^d$, the translation component, and $R\in H$, the rotation (and potentially reflection). The semi-direct product can naturally be encoded as a matrix using homogeneous coordinates
\[(t,R) \leftrightarrow \begin{pmatrix}R&t\\0 & 1\end{pmatrix},\]
so that the group product is given by a matrix product. $G$ naturally acts on a point $x\in\R^d$ through $T_{(t, R)} [x]=(t,R)x = Rx + t$.

In the experiments that we consider later in this work, we will consider the case $d=2$. In this case $\SO(2)$ has a simple description:
\[\SO(2) =\Big\{\Big(
  \begin{smallmatrix}
    \cos(\theta) &-\sin(\theta)\\
    \sin(\theta) &\cos(\theta) 
  \end{smallmatrix}\Big)\Big| \theta\in [0,2\pi)\Big\}.
\]
We will identify the groups $\Z_m$ of integers modulo $m$ with the subgroup of $\SO(2)$ given by
\[\Z_m = \Big\{\Big(
  \begin{smallmatrix}
    \cos(2\pi k/ m) &-\sin(2\pi k/m)\\
    \sin(2\pi k /m) & \cos(2\pi k /m)
  \end{smallmatrix}\Big)
  \Big|
  k\in \Z
  \Big\}.
\]


\section{Learnable equivariant maps}
\label{sec:equivariant_ops}
\label{sec:rt_equiv}
The concept of equivariance is well-suited to describing the group symmetries that a function might obey:
\begin{definition}
Given a general group $G$, a function $\Phi:\mathcal X\to \mathcal Y$ and group actions $T^{\mathcal X}, T^{\mathcal Y}$ of $G$ on $\mathcal X$ and $\mathcal Y$, $\Phi$ will be called equivariant if it satisfies
\begin{equation}
  \label{eq:general_equivariance}
  \Phi(T_g^{\mathcal X} [f]) = T_g^{\mathcal Y} [\Phi(f)]
\end{equation}
for all $f\in \mathcal X$ and $g\in G$.   
\end{definition}
Following the definition of equivariance, we see that equivariant functions have the convenient property that composing them results in an equivariant function, as long as the group actions on the inputs and outputs match in the appropriate way:
\begin{lemma}
Suppose that $G$ is a group that acts on sets $\mathcal X, \mathcal Y$ and $\mathcal Z$ through $T^{\mathcal X}, T^{\mathcal Y}$ and $T^{\mathcal Z}$. If $\Phi:\mathcal X\to\mathcal Y$ and $\Psi:\mathcal Y\to\mathcal Z$ are equivariant, then so is $\Psi\circ \Phi:\mathcal X\to\mathcal Z$.  
\end{lemma}
Based on this property it is clear that the standard approach to building neural networks (compose linear and nonlinear functions with learnable components in an alternating manner) can be used to build equivariant neural networks as long as linear and nonlinear functions with the desired equivariance can be constructed.
\begin{example}
Suppose that  $\mathcal X=L^2(\R^d, \R^{c_{\mathcal X}})$ and $\mathcal Y = L^2(\R^d, \R^{c_{\mathcal Y}})$, with the group $G=\R^d$ acting on $\mathcal X$ by $T^{\mathcal X}_h [f](x)= f(x - h)$, and in a similar way on $\mathcal Y$ by $T^{\mathcal Y}$. Ordinary CNNs~\cite{lecun_convolutional_1998}, with convolutional linear layers and pointwise nonlinear functions, are equivariant in this setting.
\end{example}

In this work, we will consider the group $G=\R^d\rtimes H$ for some subgroup $H$ of $\Orth(d)$ (see Section~\ref{sec:background} for some background), acting on vector-valued functions. To be more specific, we will let $\mathcal X=L^2(\R^d, \R^{d_{\mathcal X}})$ be the Hilbert space of square-integrable $\R^{d_{\mathcal X}}$-valued functions and assume that $\R^{d_{\mathcal X}}$ carries a representation $\pi_{\mathcal X} : H \to \GL(\R^{d_{\mathcal X}})$. Similarly, we will define $\mathcal Y=L^2(\R^d, \R^{d_{\mathcal Y}})$ and assume that $\pi_{\mathcal Y}:H\to\GL(\R^{d_{\mathcal Y}})$ is a representation of $H$. We define the group actions $T^{\mathcal X}$ and $T^{\mathcal Y}$ to be the induced representations, $\rho_{\mathcal X}$ and $\rho_{\mathcal Y}$, of $\pi_{\mathcal X}$ and $\pi_{\mathcal Y}$ on $\mathcal X$ and $\mathcal Y$ respectively. In the setting that we are considering, these representations take a particularly simple form. As mentioned in Section~\ref{sec:background}, since we assume that $G$ takes the semi-direct product form $\R^d\rtimes H$, each group element $g\in G$ can be uniquely thought of as a pair $g=(t, R)$ for some $t\in \R^d$ and $R\in H$. With this in mind, the representations $\rho_{\mathcal X}$ and $\rho_{\mathcal Y}$ can be written as follows for any $f\in\mathcal Z,x\in \R^d$ and $t\in \R^d, R\in H$:
\begin{equation}
  \label{eq:induced_repr}
  \rho_{\mathcal Z}((t, R)) [f](x) = \underbrace{\pi_{\mathcal Z}(R)}_{\text{(a)}} \underbrace{f((t, R)^{-1} x)}_{\text{(b)}}\quad\text{for}\quad \mathcal Z=\mathcal X, \,\text{or}\,\mathcal Z=\mathcal Y. 
\end{equation}
These representations have a natural interpretation: to apply a group element $(t, R)$ to a vector-valued function, we must move the vectors, as in part (b) of Equation~\eqref{eq:induced_repr}, and transform each vector accordingly, as in part (a) of Equation~\eqref{eq:induced_repr}.
\subsection{Equivariant linear operators}
It is well-established that equivariant linear operators are strongly connected to the concept of convolutions. Indeed, in a relatively general setting it has been shown that an integral operator is equivariant if and only if it is given by a convolution with an appropriately constrained kernel~\cite{cohen_general_2019}. In the setting that we are considering, the more specific result in Proposition~\ref{prop:equivariant_linear} can be derived, as done in~\cite{weiler_learning_2018,weiler_general_2019} for the case $d=2$ and~\cite{weiler_3d_2018} for the case $d=3$.
\begin{prop}
  \label{prop:equivariant_linear}
  Suppose that $\Phi : \mathcal X\to \mathcal Y$ is an operator given by integration against a continuous kernel $K:\R^d\times\R^d\to \Hom(\R^{d_{\mathcal X}}, \R^{d_{\mathcal Y}})$, \[\Phi(f)(x) = \int_{\R^d} K(x, y)f(y)\,\mathrm dy.\]
  Then the operator $\Phi$ is equivariant if and only if it is in fact given by a convolution satisfying an additional constraint: there is a continuous $k:\R^d\to \Hom(\R^{d_{\mathcal X}}, \R^{d_{\mathcal Y}})$
  \[\Phi(f)(x)=\int\limits_{\R^d} k(x - y) f(y)\,\mathrm dy,\]
  where $k$ satisfies the additional condition
  \[k(Rx) = \pi_{\mathcal Y}(R) k(x) \pi_{\mathcal X}(R^{-1})\quad\text{for}\quad x\in \R^d, R\in H.\]
\end{prop}
The derivation of this result proceeds by writing out the definitions of equivariance and using the invariances of the Lebesgue measure. The equivariance of $\Phi$ implies that we must the following chain of equalities for any $x\in \R^d, f\in \mathcal X, t\in\R^d, R\in H$ and $g=(t, R)\in G$:
\begin{align*}
  \int\limits_{\R^d} \pi_{\mathcal Y}(R) K(g^{-1} x, y) f(y)\,\mathrm dy&\overset{\text{(a)}}{=}\pi_{\mathcal Y}(R)\int\limits_{\R^d}K(g^{-1} x, y)f(x)\,\mathrm dy\\
  &=\rho_{\mathcal Y}(g)[\Phi(f)](x) \\
  &\overset{\text{(b)}}{=} \Phi(\rho_{\mathcal X}(g)[f])(x)\\
  &= \int\limits_{\R^d} K(x, y) \rho_{\mathcal X}g[f](y)\,\mathrm dy\\
  &= \int\limits_{\R^d} K(x, y)\pi_{\mathcal X} (h)f(g^{-1}y)\,\mathrm dy\\
  &\overset{\text{(c)}}{=} \int\limits_{\R^d}K(x, gy)\pi_{\mathcal X}(h) f(y)\,\mathrm dy.
\end{align*}
Here the tags above the equality signs correspond to the following justifications:
\begin{enumerate}[label={(\alph*)}]
\item Since $\pi_{\mathcal Y}$ is a group representation, $\pi_{\mathcal Y}(R)$ is a linear map and commutes with the integral,
\item $\Phi$ is assumed to be equivariant,
\item We make the substitution $y\leftarrow g y$ and note that the Lebesgue measure is invariant to $G$.
\end{enumerate}
Taking the left hand side and right hand side together, we find that
\[\int\limits_{\R^d} \Big(\pi_{\mathcal Y}(R)K(g^{-1} x, y) - K(x,g y) \pi_{\mathcal X}(R)\Big)f(y)\,\mathrm dy = 0,\]
and since this must hold for any $f\in \mathcal X=L^2(\R^d, \R^{d_{\mathcal X}})$, we conclude by testing on sequences converging to Dirac delta functions that
\begin{equation}
  \label{eq:int_kernel_restriction}
  \pi_{\mathcal Y}(R) K(g^{-1}x, y) = K(x,g y) \pi_{\mathcal X}(R).
\end{equation}
Specialising by setting $R$ equal to the identity element, we see that
\[K(x - t, y) = K((t, I)^{-1}x, y) = K(x, (t, I) y) = K(x, y + t),\]
or upon substituting $x\leftarrow x+t$, $K(x, y) = K(x + t, y + t)$. Choosing $t$ to be the translation that takes $y$ to $0$, we find that
\[K(x, y) = K(x - y, 0) =: k(x - y)\]
defines a convolution kernel $k:\R^d\to \Hom(\R^{d_{\mathcal X}}, \R^{d_{\mathcal Y}})$. Now specialising Equation~\eqref{eq:int_kernel_restriction} by letting $R\in H$ and $x\in\R^d$ be arbitrary and $t, y=0$, we obtain the condition $\pi_{\mathcal Y}(R) k(R^{-1}x) = k(x)\pi_{\mathcal X}(R)$, or upon substituting $x\leftarrow Rx$ and rearranging,
\begin{equation}
  \label{eq:conv_kernel_restriction}
  k(Rx) = \pi_{\mathcal Y}(R) k(x) \pi_{\mathcal X}(R^{-1}).
\end{equation}
Conversely, the above reasoning can be reversed to show that the condition in Equation~\eqref{eq:conv_kernel_restriction} (for all $x\in\R^d, R\in H$) is sufficient to guarantee equivariance of $\Phi$.

The condition in Equation~\eqref{eq:conv_kernel_restriction} is a linear constraint that is fully specified before training. Hence, if a basis is computed for the convolution kernels satisfying Equation~\eqref{eq:conv_kernel_restriction}, a general equivariant linear operator can be learned by learning its parameters in that basis. Since the choices of $H$ that we consider are all compact groups, any representation of $H$ can be decomposed as a direct sum of irreducible representations of $H$ (Theorem~5.2 in~\cite{folland_course_2015}). As a result of this, we can give the following procedure to compute a basis for the convolution kernels satisfying the equivariance condition in Equation~\eqref{eq:conv_kernel_restriction} as soon as $\pi_{\mathcal X}$ and $\pi_{\mathcal Y}$ are specified:
\begin{itemize}
\item Decompose $\pi_{\mathcal X}$ and $\pi_{\mathcal Y}$ as direct sum of irreducible representations; $\pi_{\mathcal X} = Q_{\mathcal X} \diag(\pi_{\mathcal X}^1,\ldots, \pi_{\mathcal X}^{k_{\mathcal X}})Q_{\mathcal X}^{-1}, \pi_{\mathcal Y} = Q_{\mathcal Y} \diag(\pi_{\mathcal Y}^1,\ldots, \pi_{\mathcal Y}^{k_{\mathcal Y}})Q_{\mathcal Y}^{-1}$ (here $\diag$ constructs a block diagonal matrix with the diagonal elements given by the arguments supplied to $\diag$).
\item For each $i, j$ with $1\leq i\leq k_{\mathcal X}, 1\leq j \leq k_{\mathcal Y}$ find a basis for the convolution kernels $k_{i,j}$ satisfying the equivariance condition
  \[k_{i,j}(Rx)=\pi_{\mathcal Y}^j (R) k_{i,j}(x) \pi_{\mathcal X}^j(R^{-1})\]
  with the irreducible representations $\pi_{\mathcal Y}^j$ and $\pi_{\mathcal X}^i$.
\item Given expansions of the $k_{i,j}$, compute the overall equivariant convolution kernel $k$ by
  \[k = Q_{\mathcal Y}\cdot(k_{i,j})_{1\leq i\leq k_{\mathcal X}, 1\leq j\leq k_{\mathcal Y}}\cdot Q_{\mathcal X}^{-1}.\]
\end{itemize}
This procedure has been described in more detail in~\cite{weiler_general_2019} and implemented in the corresponding software package for the groups $G=\R^2\rtimes H$, where $H$ can be any subgroup of $\Orth(2)$.

Since the equivariant convolutions described above are implemented using ordinary convolutions, little extra computational effort required to use them compared to ordinary convolutions: during training, there is just an additional step of computing the basis expansion defining the equivariant convolution kernels (and backpropagating through it). When it is time to test the network, this step can be avoided by computing the basis expansion once and only saving the resulting convolution kernels, so that it is completely equivalent in terms of computational effort to using an ordinary CNN.

\subsection{Equivariant nonlinearities}
Although pointwise nonlinearities are translationally equivariant, some more care is needed when designing nonlinearities that satisfy the equivariance condition in Equation~\eqref{eq:general_equivariance} with our choices of groups. Examining the form of the induced representations in our setting, as given in Equation~\eqref{eq:induced_repr}, it is evident that for a pointwise nonlinearity $\phi:\R\to\mathcal \R$ to be equivariant (in the sense that $\phi(\rho_{\mathcal X}(g)[f]) = \rho_{\mathcal X}(g)[\phi(f)]$, with $\phi$ applied pointwise) $\phi$ must commute with $\pi_{\mathcal X}(R)$ for every $R\in H$: with $g = (t,R)$ for $t\in\R^d, R\in H$ we have
\[\phi(\pi_{\mathcal X}(R) f(g^{-1}x))=\phi(\rho_{\mathcal X}(g)[f])(x) = \rho_{\mathcal X}(g)[\phi(f)](x)=\pi_{\mathcal X}(h)\phi(f(g^{-1}x)).\]
This can be ensured if $\pi_{\mathcal X}$ is the regular representation of $H$, since in that case each $\pi_{\mathcal X}(h)$ is a permutation matrix, giving the following guideline:
\begin{lemma}
  Suppose that $G=\R^d\rtimes H$ with $H$ a finite subgroup of $\Orth(d)$ and that $\phi:\R\to\R$ is a given function. If $\pi_{\mathcal X}$ is the regular representation of $H$, then $\Phi:\mathcal X\to\mathcal X$ is equivariant, where $\Phi(f)(x) = \phi(f(x))$.
\end{lemma}
Another way to ensure that $\phi$ commutes with $\pi_{\mathcal X}$ is by choosing the trivial representation. Although the trivial representation may not be very interesting by itself, this gives rise to another form of nonlinearity called the norm nonlinearity. If $\pi_{\mathcal X}$ is a unitary representation, taking the pointwise norm satisfies an equivariance condition: with $g=(t,R)$ for $t\in \R^d, R\in H$
\[\|\rho_{\mathcal X}(g)[f](x)\| = \|\pi_{\mathcal X}(R) f(g^{-1}x))\| = \|f(g^{-1}x)\|.\]
The right-hand side transforms according to the trivial representation, so by the above comments we deduce that the nonlinearity $f \mapsto \phi(\|f\|)$ satisfies an equivariance condition of the same form. To obtain the norm nonlinearity, which maps features of a given type to features of the same type, we then form the map $\Phi:\mathcal X\to\mathcal X, f\mapsto f\cdot \phi(\|f\|)$: with $g=(t, R)$ for $t\in\R^d, R\in H$, we have
\begin{align*}
  \Phi(\rho_{\mathcal X}(g)[f])(x) &= \pi_{\mathcal X}(R)f(g^{-1}x)\cdot\phi( \|f(g^{-1}x)\|)\\
  &=\pi_{\mathcal X}(R) \Big(f(g^{-1}x)\cdot \phi(\|f(g^{-1}x)\|)\Big)\\
  &=\pi_{\mathcal X}(R) \Big(f\cdot \phi(\|f\|)\Big)(g^{-1}x)\\
  &= \rho_{\mathcal X}(g) [\Phi(f)](x),
\end{align*}
where we used that $\phi(\|f(g^{-1}x)\|)$ is a scalar. This shows that the norm nonlinearity $\Phi$ is indeed equivariant:
\begin{lemma}
  Suppose that $\pi_{\mathcal X}$ is a unitary representation of $H$, and that $\phi:\R\to\R$ is a given function. Then the norm nonlinearity $\Phi:\mathcal X\to \mathcal X$ with $\Phi(f)[x] = f(x)\phi(\|f(x)\|)$ is equivariant.
\end{lemma}


\section{Reconstruction methods motivated by variational regularisation}
\label{sec:methods_var_reg}
We consider the inverse problem of estimating an image $u$ from noisy measurements $y$. We will assume that knowledge of the measurement process is available in the form of the forward operator $A$, which maps an image to ideal, noiseless measurements, and generally there were will be a reasonable idea of the process by which they are corrupted to give rise to the noisy measurements $y$. A tried and tested approach to solving inverse problems is the variational regularisation approach~\cite{engl_regularization_1996,burger_convergence_2004}. In this approach, images are recovered from measurements by minimising a trade-off between the data fit and a penalty function encoding prior knowledge:
\begin{equation}\hat u = \argmin_u E(u) + J(u),\label{eq:variational_regularisation}\end{equation}
with $E$ a data discrepancy functional penalising mismatch of the estimated image and the measurements and $J$ the penalty function. Usually $E$ will take the form $E(u) = d(A(u), y)$, where $d$ is a measure of divergence chosen based on our knowledge of the noise process.
\subsection{Equivariance in splitting methods}
Generally, Problem~\eqref{eq:variational_regularisation} may be difficult to solve, and a lot of research has been done on methods to solve problems such as these. Iterative methods to solve it are often structured as splitting methods: the objective function is split into terms, and easier subproblems associated with each of these terms are solved in an alternating fashion to yield a solution to Problem~\eqref{eq:variational_regularisation} in the limit. A prototypical example of this is the proximal gradient method (also known as forward-backward splitting)~\cite{bruck_weak_1977,passty_ergodic_1979}, which has become a standard tool for solving linear inverse problems, particularly in the form of the FISTA algorithm~\cite{beck_fast_2009}. In its basic form, the proximal gradient method performs the procedure described in Algorithm~\ref{alg:fb}.
\begin{algorithm}
  \caption{Proximal gradient method}
  \label{alg:fb}
  \begin{algorithmic}
    \State $u\leftarrow u^0$
    \For{$i\leftarrow 1,\ldots, \texttt{it}$}
    \State $u\leftarrow \prox_{\tau^i J}(u - \tau ^i \nabla_u E(u))$
    \EndFor\\
    \Return $u$
  \end{algorithmic}
  \end{algorithm}

  Recall here that the proximal operator~\cite{moreau_fonctions_1962,moreau_proprietes_1963, moreau_proximite_1965} $\prox_{J}$ is defined as follows:
  \begin{definition}
    \label{def:prox_definition}
    Suppose that $\mathcal X$ is a Hilbert space and that $J:\mathcal X\to \R\cup \{+\infty\}$ is a lower semi-continuous convex proper functional. The proximal operator $\prox_J:\mathcal X\to \mathcal X$ is then defined as
    \begin{equation}
      \label{eq:prox_definition}
      \prox_J(u) = \argmin_{u'\in\mathcal X} \frac{1}{2}\|u - u'\|^2 + J(u')
    \end{equation}
  \end{definition}

  Although this definition of proximal operators assumes that the functional $J$ is convex, this assumption is more stringent than is necessary to ensure that an operator defined by Equation~\eqref{eq:prox_definition} is well-defined and single-valued. One can point for example to the classes of $\mu$-semi-convex functionals (i.e. the set of $J$, such that $u\mapsto J(u) +\frac{\mu}{2}\|u\|^2$ is convex) on $\mathcal X$ for $0<\mu < 1$, which include nonconvex functionals. In what follows, we will allow for such more general functionals by just asking that the proximal operator is well-defined and single-valued.
  
    It is often reasonable to ask that the proximal operators $\prox_{\tau J}$ satisfy an equivariance property; if the corresponding regularisation functional $J$ is invariant to a group symmetry, the proximal operator will be equivariant:
    \begin{prop}
      \label{prop:invariant_regulariser_implies_equivariant_prox}
      Suppose that $\mathcal X$ is a Hilbert space and $\rho$ is a unitary representation of a group $G$ on $\mathcal X$. If a functional $J:\mathcal X\to \R\cup \{+\infty\}$ is invariant, i.e.~$J(\rho(g) f) = J(f)$, and has a well-defined single-valued proximal operator $\prox_J:\mathcal X\to\mathcal X$, then $\prox_J$ is equivariant, in the sense that
      \[\prox_J(\rho(g) f) = \rho(g)\prox_J(f)\]
      for all $f\in \mathcal X$ and $g\in G$. 
    \end{prop}
    \begin{proof}
      We have the following chain of equalities:
\begin{align*}
  \prox_J(\rho(g)f) &= \argmin_{h} \frac{1}{2} \|\rho(g)f - h \|^2 + J(h)\\
                   &\overset{\text{(a)}}{=} \argmin_h \frac{1}{2} \|\rho(g) (f - \rho(g^{-1}) h)\|^2 + J(\rho(g^{-1}) h)\\
                   &\overset{\text{(b)}}{=} \argmin_h \frac{1}{2} \|f - \rho(g^{-1}) h\|^2 + J(\rho(g^{-1})h)\\
  &\overset{\text{(c)}}{=} \rho(g) [\argmin_h \frac{1}{2} \|f - h\|^2 + J(h)] = \rho(g) \prox_J(f).
\end{align*}
The three marked steps are justified as follows:
\begin{enumerate}[label={(\alph*)}]
\item $J$ is assumed to be invariant w.r.t.~$\rho$,
\item The representation $\rho$ is assumed to be unitary,
\item $\rho(g)$ is invertible, and under the substitution $h\leftarrow \rho(g) h$, the minimiser transforms accordingly.
\end{enumerate}
\end{proof}
\begin{example}
As a prominent example of a regularisation functional satisfying the conditions of Proposition~\ref{prop:invariant_regulariser_implies_equivariant_prox}, consider the total variation functional~\cite{rudin_nonlinear_1992} on $L^2(\R^d)$
\[\TV(u) = \sup_{\phi\in C_c^\infty(\R^d;\R^d), \| \phi\|_\infty\leq 1}\int\limits_{\R^d} u \divergence \phi, \]
with the group $G=\SE(d)$ and the scalar field representation $\rho(r)[f](x) = f(r^{-1}x)$. Since the Lebesgue measure is invariant to $G$ and the set of vector fields $\{\phi \in C^\infty_c(\R^d;\R^d) | \|\phi\|_\infty\leq 1\}$ is closed under $G$, $\TV$ is invariant w.r.t.~$\rho$. As a result of this, Proposition~\ref{prop:invariant_regulariser_implies_equivariant_prox} tells us that $\prox_{\tau \TV}$ is equivariant w.r.t.~$\rho$ for any $\tau \geq 0$. Note that $\TV$ is not unique in satisfying these conditions; by a similar argument it can be shown, for example, that the higher order total generalised variation functionals~\cite{bredies_total_2010} share the same invariance property (and hence also that their proximal operators are equivariant).
\end{example}
\begin{remark}
  \label{remark:vector_valued}
The above example, and all other examples that we consider in this work, are concerned with the case where the image to be recovered is a scalar field. Note, however, that Proposition~\ref{prop:invariant_regulariser_implies_equivariant_prox} is not limited to this type of field and that there are applications where it is natural to use more complicated representations $\rho$. A notable example is diffusion tensor MRI~\cite{coulon_diffusion_2004} in which case the image to be estimated is a diffusion tensor field and $\rho$ should be chosen as the appropriate tensor representation.
\end{remark}
\subsubsection{Equivariance of the reconstruction operator}
It is worth thinking about whether it is sensible to ask that the overall reconstruction method is equivariant, and how this should be interpreted. Thinking of the reconstruction operator as a map from measurements $y$ to images $\hat u$, it is hard to make sense of the statement that it is equivariant, since the measurement space generally does not share the symmetries of the image space (in the case where measurements may be incomplete). If we think instead of the reconstruction method as mapping a true image $u$ to an estimated image $\hat u$ through (noiseless) measurements $y= A(u)$, we might ask that a symmetry transformation of $u$ should correspond to the same symmetry transformation of $\hat u$. In the case of reconstruction by a variational regularisation method as in Problem~\eqref{eq:variational_regularisation}, this is too much to ask for even if the regularisation functional is invariant, since information in the (incomplete) measurements can appear or disappear under symmetry transformations of the true image. An example of this phenomenon when solving an inpainting problem is shown in Figure~\ref{fig:var_reg_is_not_equivariant}.
\begin{figure}[!ht]
    \centering
    \includegraphics[scale=1]{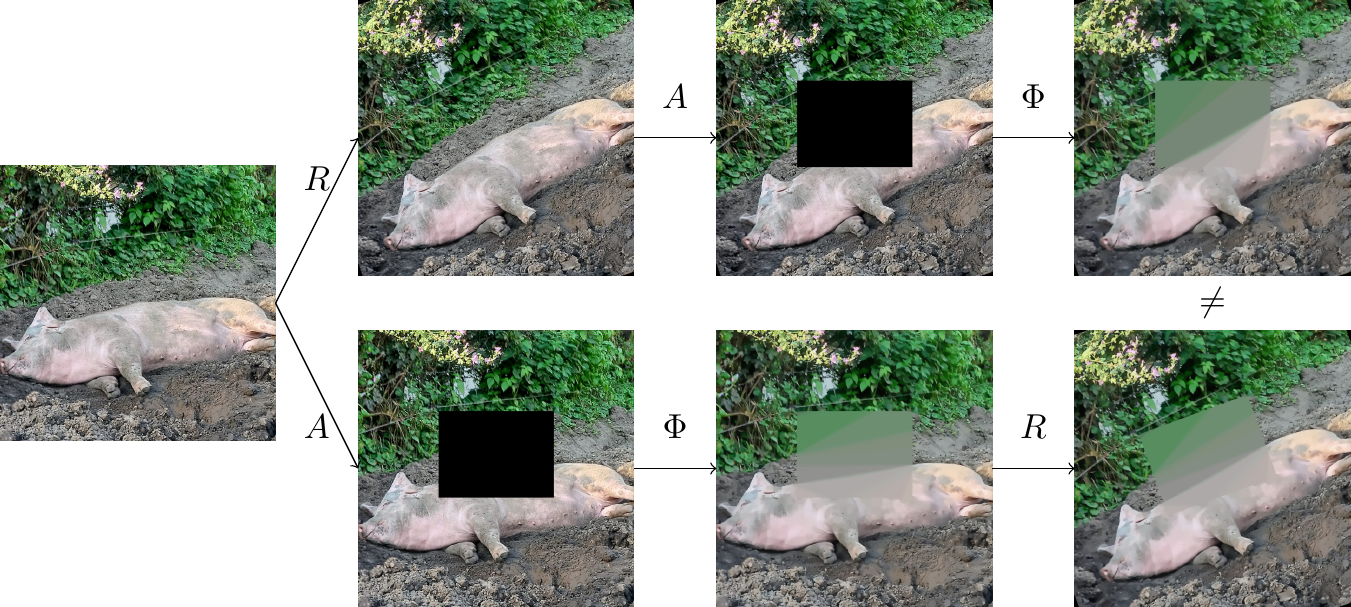}
    \caption{An example demonstrating the non-equivariance of a general variational regularisation approach to image reconstruction, even when the corresponding regularisation functional $J$ (as in Problem~\eqref{eq:variational_regularisation}) is invariant. Here, $A$ represents the application of an inpainting mask, $R$ is an operator rotating the image by $20^\circ$ and $\Phi$ is the solution map to Problem~\eqref{eq:variational_regularisation} with $E(u) = \|Au - y\|_2^2$ and $J(u) = \tau\TV(u)$.}
    \label{fig:var_reg_is_not_equivariant}
\end{figure}
\subsection{Learned proximal gradient descent}
A natural way to use knowledge of the forward model in a neural network approach to image reconstruction is in the form of unrolled iterative methods~\cite{adler_solving_2017,putzky_recurrent_2017}. Starting from an iterative method to solve Problem~\eqref{eq:variational_regularisation}, the method is truncated to a fixed number of iterations and some of the steps in the truncated algorithm are replaced by learnable parts. As noted in the previous section, the proximal gradient method in Algorithm~\ref{alg:fb} can be applied to a variational regularisation problem such as Problem~\eqref{eq:variational_regularisation}. Motivated by this and the unrolled iterative method approach, we can study learned proximal gradient descent as in Algorithm~\ref{alg:lfb} (where the variable $s$ can be used as a memory state as is common in accelerated versions of the proximal gradient method~\cite{beck_fast_2009}):
    \begin{algorithm}
      \caption{Learned proximal gradient method}
      \label{alg:lfb}
      \begin{algorithmic}
        \State $u \leftarrow u^0, s\leftarrow 0$
        \For{$i\leftarrow 1,\ldots, \texttt{it}$}
        \State $(u, s) \leftarrow \widehat\prox_i(u, s, \nabla E(u))$
        \EndFor\\
        \Return $\Phi(y):=u$
      \end{algorithmic}
    \end{algorithm}
    
Here $\widehat \prox_i$ are neural networks, the architectures of which are chosen to model proximal operators. In this work, we choose $\widehat \prox_i$ to be defined as
\begin{equation}
    \widehat \prox_i = K_{\text{project}, i}\circ (\id + \phi\circ K_{\text{intermediate}, i}) \circ K_{\text{lift}, i}, \label{eq:define_prox_model}
\end{equation}
where each of the $K_{\text{project}, i}, K_{\text{intermediate}, i}$ and $K_{\text{lift}, i}$ are learnable affine operators and $\phi$ is an appropriate nonlinear function. We can appeal to Proposition~\ref{prop:invariant_regulariser_implies_equivariant_prox} and model $\widehat\prox_i$ as translationally equivariant (we will call the corresponding reconstruction method the ordinary method in what follows) or as roto-translationally equivariant (we will call the corresponding reconstruction method the equivariant method in what follows). 

\begin{figure}[!htb]
  \centering
  \includegraphics[scale=1]{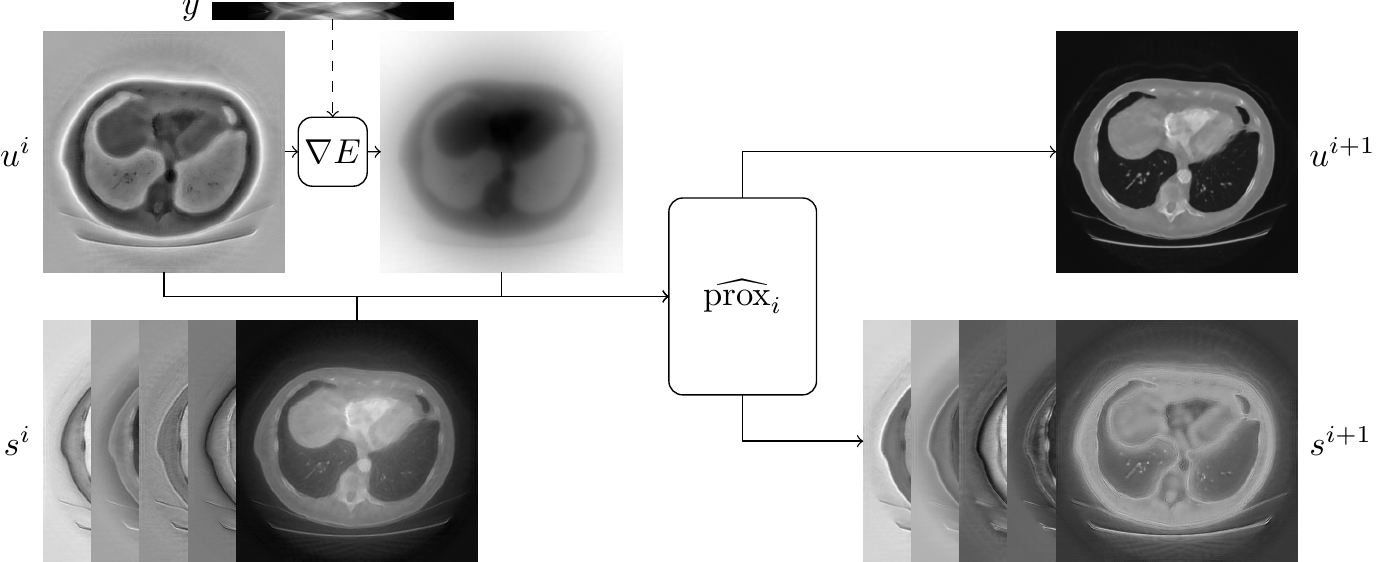}
  \caption{A schematic illustration of a single iteration of the learned proximal gradient method, Algorithm~\ref{alg:lfb}, for a CT reconstruction problem. The choice of $E$ is described in Section~\ref{sec:lidc}. Knowledge of the forward model is incorporated into the reconstruction through $\nabla E$, which is not an equivariant operator in general. Motivated by Proposition~\ref{prop:invariant_regulariser_implies_equivariant_prox}, we know that $\widehat{\prox_i}$ is naturally modelled as an equivariant operator.
  }
  \label{fig:prox_block}
\end{figure}

Recall that we consider groups of the form $G=\R^d\rtimes H$ for subgroups $H$ of $\Orth(d)$ in this work. Since we apply the learned equivariant method to reconstruct scalar-valued images, the input and output types of each $\widehat{\prox}_i$ should correspond to features carrying the trivial representation of $H$. For the equivariant method, $K_{\text{lift}, i}$ are equivariant convolutions from a small number of input channels with the trivial representation of $H$ to a larger number of intermediate channels with the regular representation of $H$, if $H$ is a finite group, or various irreducible representations of $H$, if $H$ is a continuous group. $K_{\text{intermediate}, i}$ are chosen as equivariant convolutions mapping the output channels of $K_{\text{lift}, i}$ to a set of channels of the same type. Finally, $K_{\text{project}, i}$ are chosen as equivariant convolutions that map the output channels of $K_{\text{intermediate}, i}$ to a small number of output channels with the trivial representation of $H$.

For the ordinary method, $K_{\text{lift}, i}$ are ordinary convolutions mapping a small number of input channels to a larger number of intermediate channels, $K_{\text{intermediate}, i}$ are ordinary convolutions mapping the output channels of $K_{\text{lift}, i}$ to a set of channels of the same type, and $K_{\text{project}, i}$ are ordinary convolutions mapping the many output channels of $K_{\text{intermediate}, i}$ to a small number of output channels.

Since the implementations of the equivariant convolutions are ultimately based on ordinary convolutions, a natural comparison can be made between the equivariant and ordinary method by matching the widths of the underlying ordinary convolutions. When the methods are compared in this way, they should take comparable computational effort to use and the ordinary method is a superset of the equivariant method in the sense that the parameters of the ordinary method can be chosen to reproduce the action of the equivariant method.

\begin{remark}
Both in the case of Algorithm~\ref{alg:fb} and Algorithm~\ref{alg:lfb}, we require access to the gradient $\nabla E$, where $E$ is a data discrepancy functional. In our case, $E$ always takes the form $E(u) = d(A(u), y)$ where $A$ is the forward operator and $d$ is a measure of divergence. As a result of this $E$ can be differentiated by the chain rule as long as we have access to the gradient of $d$ and can compute vector-Jacobian products of $A$. If the forward operator $A$ is linear, its vector-Jacobian products are just given by the action of the adjoint of $A$.  
\end{remark}

\section{Experiments}
\label{sec:experiments}
In this section, we demonstrate that roto-translationally equivariant operations can be incorporated into a learned iterative reconstruction method such as Algorithm~\ref{alg:lfb} to obtain higher quality reconstructions than those obtained using comparable reconstruction methods that only use translationally equivariant operations. We consider two different inverse problems: a subsampled MRI problem and a low-dose CT problem. The code that was used to produce the experimental results shown is freely available at \url{https://github.com/fsherry/equivariant\_image\_recon}.
\subsection{Datasets}
\label{sec:datasets}
  \subsubsection{LIDC-IDRI dataset}
  \label{sec:lidc}
We use a selection of chest CT images of size $512\times 512$ from the LIDC-IDRI dataset~\cite{armato_iii_lung_2011,armato_iii_data_2015} for our CT experiments. As in Section~\ref{sec:fastmri}, we screen the images to remove as many low-quality images as possible, The set is split into 5000 images that can be used for training, 200 images that can be used for validation and 1000 images that can be used for testing. For the experiments using this dataset, we use the ASTRA toolbox~\cite{palenstijn_performance_2011,van_aarle_astra_2015,aarle_fast_2016} to simulate a parallel beam ray transform $\mathcal R$ with 50 uniformly spaced views at angles between $0$ and $\pi$. We simulate the measurements $y$ as post-log data in a low-dose setting:
\[y = -\frac{1}{\mu}\log\Big(\max\Big\{\frac{n}{N_\text{in}}, \eta\Big\}\Big),\quad\text{where}\quad n\sim\Pois(N_\text{in} \exp(-\mu \mathcal R(u))).\]
Here $N_\text{in}=10000$ is the average number of photons per detector pixel (without attenuation), $\mu$ is a base attenuation coefficient connecting the volume geometry and attenuation strength, and $\eta$ is a small constant to ensure that the argument of the logarithm is strictly positive, chosen as $\eta=10^{-8}$ in our experiments. In these experiments, we will define the data discrepancy functional $E$ as
\[E(u) = \frac{1}{2}\|\mathcal R u - y\|_2^2.\]
\begin{figure}[!htb]
  \centering
  \includegraphics[scale=1]{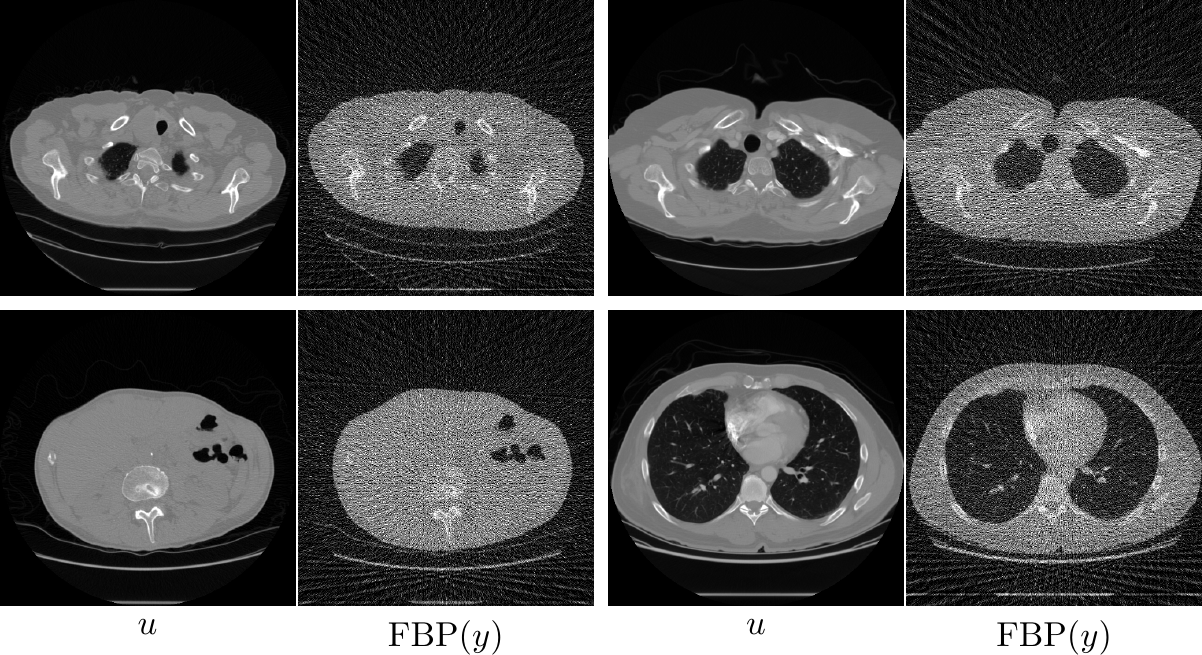}
  \caption{Four samples of the images that were used to train the reconstruction operators in the CT experiments, and the results of applying filtered backprojection (FBP) to the corresponding simulated sinograms. The images are clipped between upper and lower attenuation coefficient limits of $-1024$~HU and $1023$~HU.}
  \label{fig:ct_samples}
\end{figure}
\subsubsection{FastMRI}
\label{sec:fastmri}
  We use a selection of axial T1-weighted brain images of size $320\times 320$ from the FastMRI dataset~\cite{knoll_fastmri_2020,zbontar_fastmri_2019} for our MRI experiments. We use a combination of $L^1$ norm and the $\TV$ functional as a simple way to screen out low-quality images. The details of this procedure can be found in the code repository associated with this work. The set is split into 5000 images that can be used for training, 200 images that can be used for validation and 1000 images that can be used for testing. For the experiments using this dataset, we simulate the measurements using a discrete Fourier transform $\mathcal F$ and a variable density Cartesian line sampling pattern $\mathcal S$ (simulated using the software package associated with the work in~\cite{lustig_sparse_2007} and shown in Figure~\ref{fig:mri_samples}):
  \[y = \mathcal S\mathcal F u + \eps,\]
  where $\eps$ is complex-valued white Gaussian noise. In this setting, a complex-valued image is modeled as a real image with two channels, one for the real part and the other for the imaginary part. The corresponding data discrepancy functional ($E$ in Equation~\eqref{eq:variational_regularisation}) will be defined as
  \[E(u) = \frac{1}{2}\|\mathcal S\mathcal F u - y\|_2^2.\]
  \begin{figure}[!htb]
    \centering
    \includegraphics[scale=1]{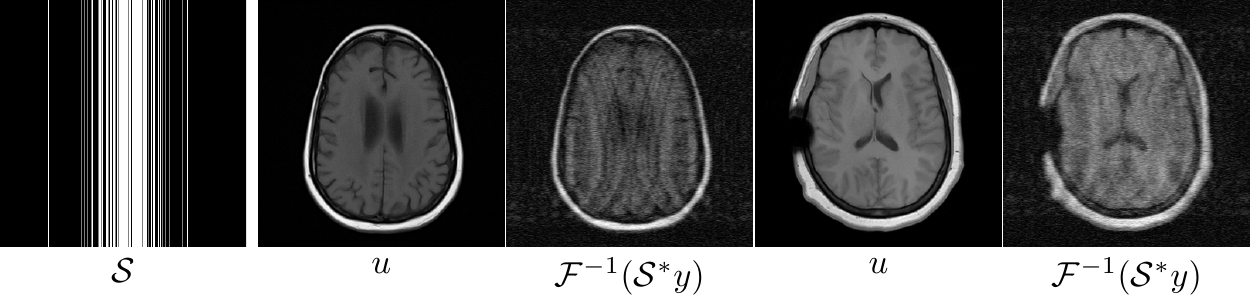}
    \caption{The sampling mask $\mathcal S$ used in the MRI experiments, sampling $20.3\%$ of k-space, and two samples of the images that were used to train the reconstruction operators in the MRI experiments, and the zero-filling reconstructions from the corresponding simulated k-space measurements.}
    \label{fig:mri_samples}
  \end{figure}
\subsection{Experimental setup}
\subsubsection{Learning framework}
Although it is also possible to learn the parameters of the reconstruction methods in Algorithm~\ref{alg:lfb} in an unsupervised learning setting, all experiments that we consider in this work can be classified as supervised learning experiments: given a finite training set $\{(u_i, y_i)\}_{i=1}^N$ of ground truth images $u_i$ and corresponding noisy measurements $y_i$, we choose the parameters of $\Phi$ in Algorithm~\ref{alg:lfb} by solving the empirical risk minimisation problem
\[\min_\Phi \frac{1}{N} \sum\limits_{i=1}^N \| u_i - \Phi(y_i)\|_2^2.\]
\subsubsection{Architectures and initialisations of the reconstruction networks}
To ensure fair comparisons between the various methods that we compare, we fix as many as possible of the aspects of the methods that are orthogonal to the point investigated in the experiments. To this end, every learned proximal gradient method has a depth of $\texttt{it}=8$ iterations. Both for the CT and MRI experiment, the images being recovered are two-dimensional, so we use equivariant convolutions with respect to groups of the form $\R^2\rtimes \Z_m$. Since the equivariant convolutions are implemented using ordinary convolutions, it is natural and straightforward to compare methods with the same width. The width of each network is the same (feature vectors that transform according to the regular representation take up $|H|$ ``ordinary'' channels, and we fix the size of the product $|H|\cdot n_\text{channels}=96$ where $n_\text{channels}$ is the number of such feature vectors in the intermediate part of $\widehat{\prox}_i$ in Equation~\eqref{eq:define_prox_model}). All convolution filters used are of size $3\times 3$. We choose the initial reconstruction $u^0=0$ and use a memory variable $s$ of five scalar channels wide in the learned proximal gradient method (Algorithm~\ref{alg:lfb}).

Furthermore we ensure that the initialisation of both types of methods are comparable. Referring back to Equation~\eqref{eq:define_prox_model}, we choose to initialise $K_{\text{intermediate}, i}$ equal to zero and let $K_{\text{project}, i}$ and $K_{\text{lift}, i}$ be randomly initialised using the He initialisation method~\cite{he_delving_2015}, as implemented in PyTorch~\cite{paszke_pytorch_2019} for ordinary convolutions and generalised to equivariant convolutions in~\cite{weiler_learning_2018} and implemented in the software package \url{https://github.com/QUVA-Lab/e2cnn}~\cite{weiler_general_2019}.
\subsubsection{Hyperparameters of the equivariant methods}
\label{sec:hyperparameters}
In addition to the usual parameters of a convolutional neural network, the learned equivariant reconstruction methods have additional parameters related to the choice of the symmetry group its representations to use. In this work, we have chosen to work with groups of the form $\R^2\rtimes \Z_m$, so a choice needs to be made which $m\in\N$ to consider.

In Figure~\ref{fig:vary_group_order}, we see the result of training and validating learned equivariant reconstruction methods on the CT reconstruction problem, with various orders $m$ of the group $H=\Z_m$. Each of the learned methods is trained on the same training set consisting of 100 images. The violin plots used give kernel density estimates of the distributions of the performance measures; for each one, we have omitted the top and bottom 5\% of values so as not to be misled by outliers. Evidently, in this case, the groups of on-grid rotations significantly outperform the other choices, with $m=4$ giving the best performance. Based on this result, all further experiments with the equivariant methods will use the group $H=\Z_4$.
\begin{figure}[!htb]
  \centering
  \includegraphics[scale=1]{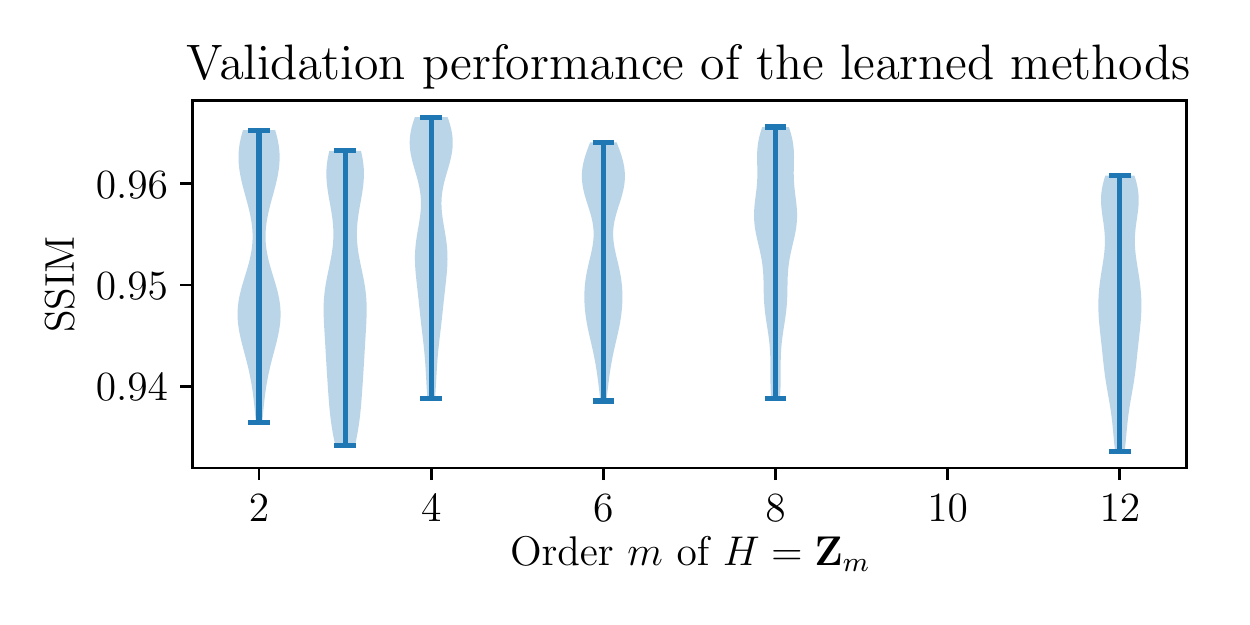}
  \caption{The reconstruction quality, as measured on a validation set, of learned proximal gradient methods trained on the CT reconstruction problem with varying orders of the group $H$. Note that when $H$ is chosen to represent on-grid rotations (i.e.~$m=2$ or $m=4$), the performance is significantly better than for any of the other choices of $H$.}
  \label{fig:vary_group_order}
\end{figure}
\subsubsection{Training details}
For both the equivariant and ordinary reconstruction methods, we train the methods using the Adam optimisation algorithm~\cite{kingma_adam_2017} with learning rate $10^{-4}$, $\beta_1=0.9, \beta_2=0.999$ and $\eps=10^{-8}$. We use minibatches of size 1 and perform a total of $10^5$ iterations of the Adam algorithm to train each method. Since we have chosen to use the finite group approach, with intermediate fields transforming according to their regular representation, we can use a pointwise nonlinearity for both the equivariant and ordinary reconstruction methods. In all experiments, we use the leaky ReLU function as the nonlinearity ($\phi$ in Equation~\eqref{eq:define_prox_model}), applied pointwise:
\[\phi(x) =
  \begin{cases}
    x\quad&\text{if}\quad x>0,\\
    0.01x\quad &\text{else.}
  \end{cases}
\]
Each training run is performed on a computer with an Intel Xeon Gold 6140 CPU and a NVIDIA Tesla P100 GPU. Training the equivariant methods requires slightly more computational effort than the ordinary methods: to begin with, given the specification of the architecture, bases need to be computed for the equivariant convolution kernels (this takes negligible effort compared to the effort expended in training). Besides this, each training iteration requires the computation of the convolutional filter from its parameters and the basis functions and the backpropagation through this basis expansion. To give an example of the extra computational effort required, we have timed 100 training iterations for comparable equivariant and ordinary methods for the MRI reconstruction problem: this took 35.5 seconds for the ordinary method and 41.9 seconds for the equivariant method, an increase of 18\%. Note that at test time, however, the ordinary and equivariant methods can be computed with the same effort.
\subsection{CT experiment: varying the size of the training set}
\label{sec:ct_rot}
In this experiment, we study the effect of varying the size of the training set on the performance of the equivariant and ordinary methods. We consider a range of training set sizes, as shown in Figure~\ref{fig:ct_rot_ssims}, and test the learned reconstruction methods on images that were not seen during training time, both in the same orientation and randomly rotated images. The violin plots displayed have the same interpretation as those shown in Figure~\ref{fig:vary_group_order} and described in Section~\ref{sec:hyperparameters}. From this comparison, we see that the equivariant method is able to better take advantage of smaller training sets than the ordinary method. Furthermore, we see that the equivariant method performs roughly equally well regardless of the orientation of the images, whereas the performance of the ordinary method drops when testing on rotated images. Figure~\ref{fig:ct_rot_recons} shows some examples of test reconstructions made with the methods learned on a training set of size $N=100$. In these reconstructions, it can be seen that the equivariant method does better at removing streaking artefacts than the ordinary method.
\begin{figure}[!htb]
  \centering
  \label{fig:ct_rot_ssims}
  \includegraphics[scale=1]{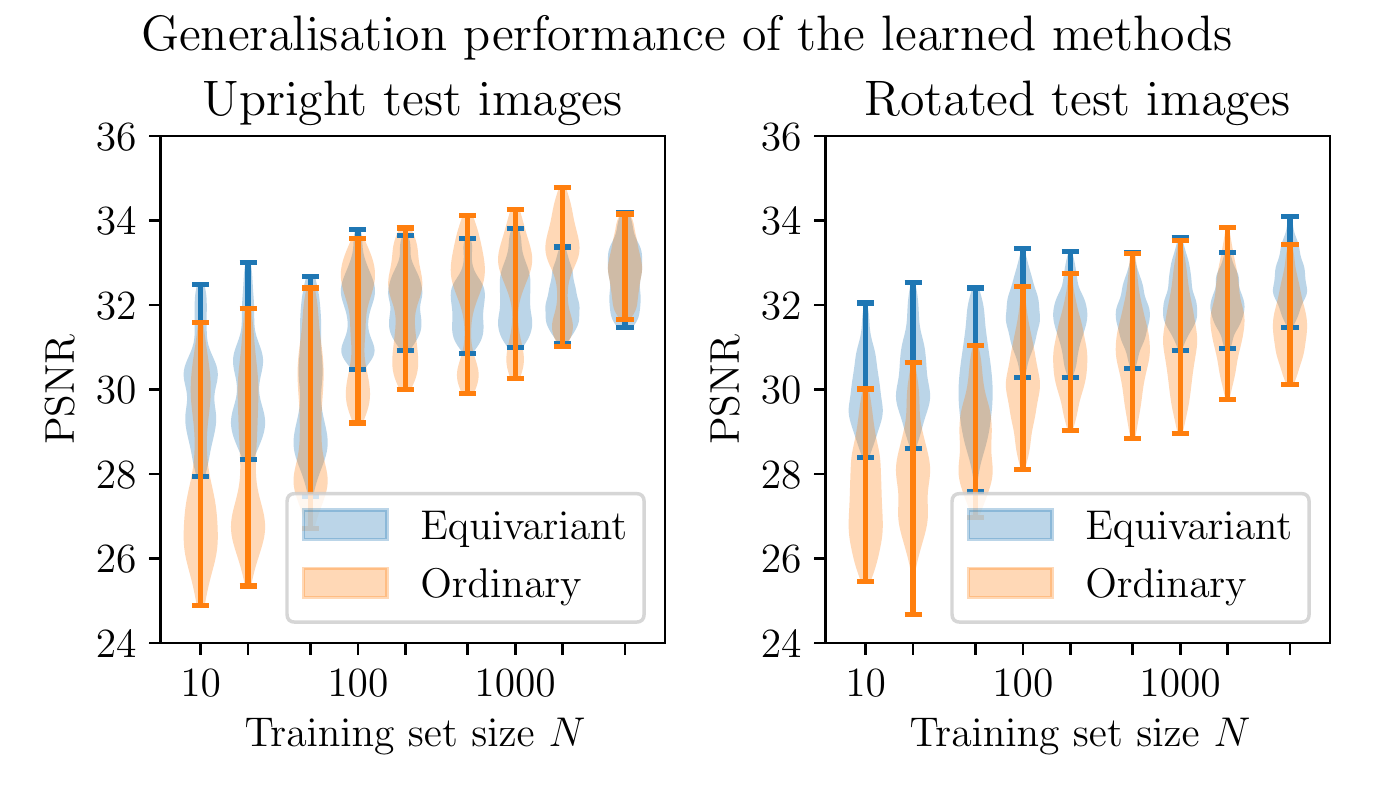}
    \caption{A comparison of the performance of equivariant and ordinary learned proximal gradient methods trained on training sets of various sizes for the CT reconstruction problem. The methods are tested on images that have not been seen during training time, both in the same orientations as were observed during training (``Upright test images'') and rotated at random angles (``Rotated test images'').}
\end{figure}

\begin{figure}[!htb]
  \centering
  \includegraphics[scale=1]{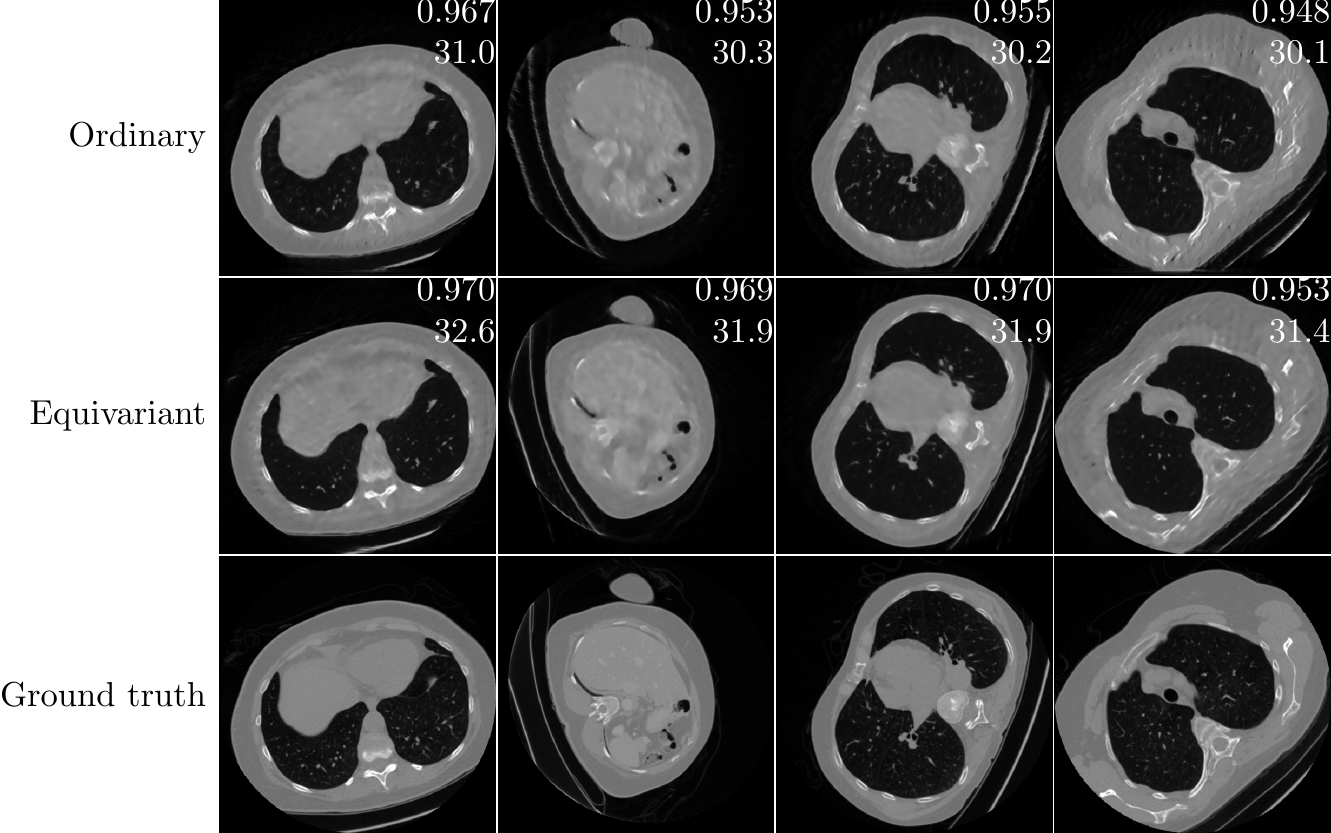}
  \caption{A random selection of test images corresponding to the plots shown in Figure~\ref{fig:ct_rot_ssims}, with a training set of size $N=100$. On each reconstruction, the top number is its SSIM and the bottom number is its PSNR w.r.t. the ground truth. The images are clipped between upper and lower attenuation coefficient limits of -1024~HU and 1023~HU.}
  \label{fig:ct_rot_recons}
\end{figure}
\subsection{MRI experiment: varying the size of the training set}
\label{sec:mri_rot}
This experiment is similar to the experiment in Section~\ref{sec:ct_rot}, but concerns the MRI reconstruction problem. A notable difference with the CT reconstruction problem is that, as a result of the Cartesian line sampling pattern, the forward operator is now less compatible with the rotational symmetry. Regardless of this, we have seen in Section~\ref{sec:methods_var_reg} that it is still sensible in this context to use equivariant neural networks in a method motivated by a splitting optimisation method. The performance differential between the equivariant and ordinary methods is more subtle than in the CT reconstruction problems. In Figure~\ref{fig:mri_rot_ssims}, we see that the equivariant method can again take better advantage of smaller training sets and is more robust to images dissimilar to those seen in training. Figure~\ref{fig:mri_rot_recons} shows examples of reconstructions made with the methods learned on a training set of size $N=50$.
\begin{figure}[!htb]
  \centering
  \includegraphics[scale=1]{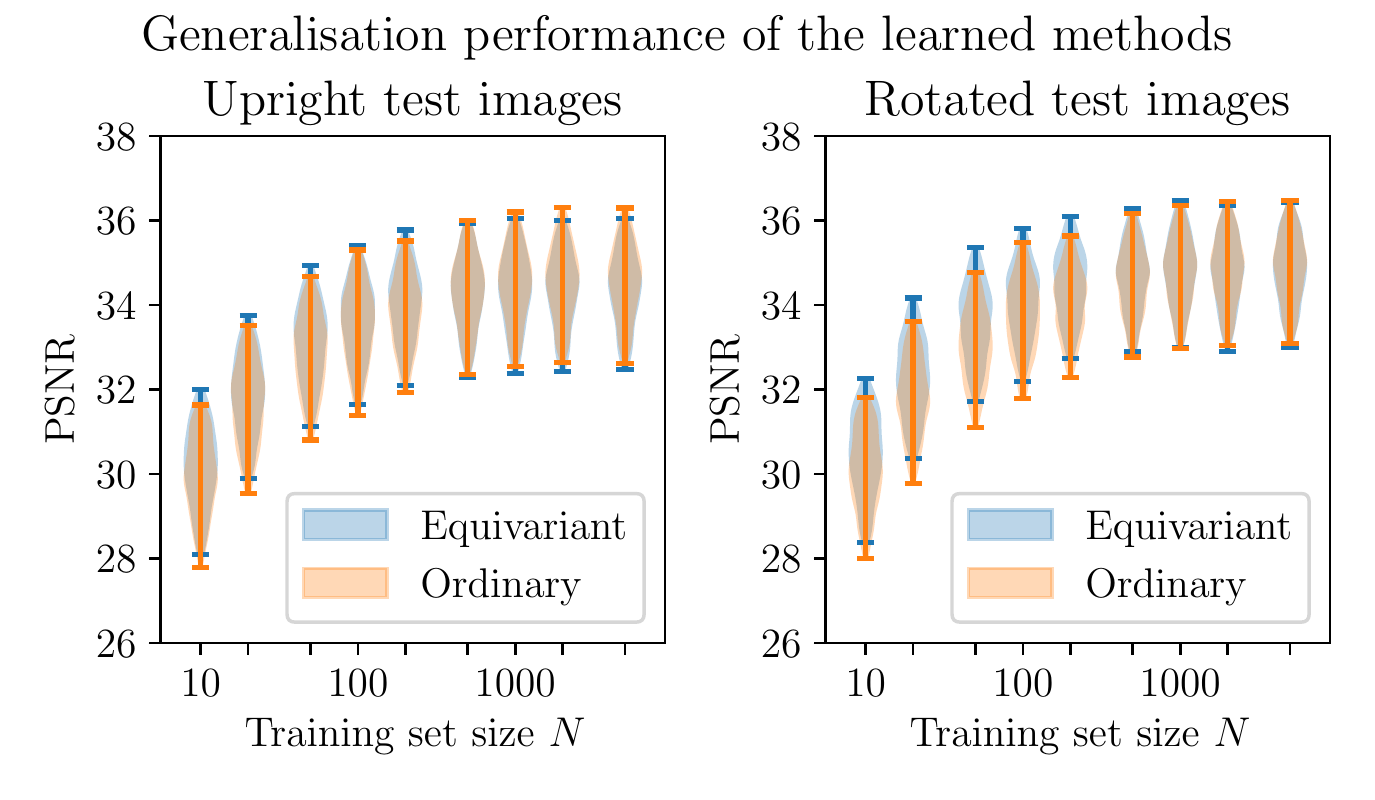}  
  \caption{A comparison of the performance of equivariant and ordinary learned proximal gradient methods trained on training sets of various sizes for the MRI reconstruction problem. The methods are tested on images that have not been seen during training time and that have been rotated at random angles.}
  \label{fig:mri_rot_ssims}
\end{figure}
\begin{figure}[!htb]
  \centering
  \includegraphics[scale=1]{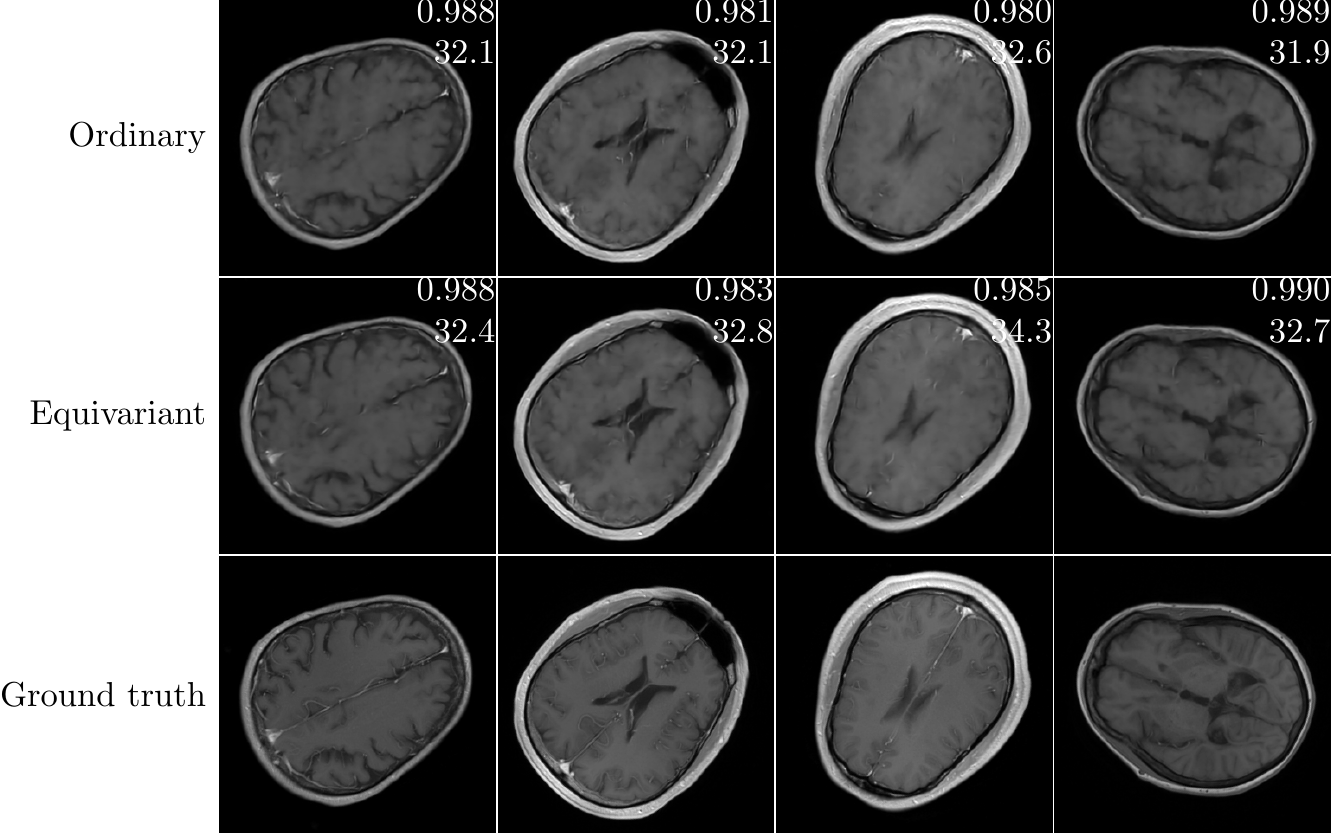}
  \caption{A random selection of test images corresponding to the plots shown in Figure~\ref{fig:mri_rot_ssims}, with a training set of size $N=50$. On each reconstruction, the top number is its SSIM and the bottom number is its PSNR w.r.t. the ground truth.}
  \label{fig:mri_rot_recons}
\end{figure}


\section{Conclusions and Discussion}
In this work, we have shown that equivariant neural networks can be incorporated into learnable reconstruction methods for inverse problems, and that doing this in a principled way results in higher quality reconstructions with little extra effort compared to ordinary convolutional neural networks. Using roto-translationally equivariant neural networks as opposed to ordinary convolutional neural networks results in better performance when trained on smaller training sets and more robustness to rotations.

In Section~\ref{sec:hyperparameters}, we saw that that the learned methods perform best when the group $H$ is chosen to be a group of on-grid rotations. In theory, one would expect better performance with a larger number of rotations, but in practice there is the issue of how the equivariant kernels are discretised. Indeed, when solving the constraint for equivariance in Equation~\eqref{eq:conv_kernel_restriction}, the allowed kernels turn out to be circular harmonics multiplied by an arbitrary radial profile, and in practice we discretise these functions on $3\times 3$ filters. An opportunity for future work on the use of equivariant neural networks can be found in how the combination of group and discretisation should be optimised.

All of the experiments shown in this work have dealt with two-dimensional images, but the methods described here can be applied equally well to three-dimensional images, as long as the two-dimensional equivariant convolutions are replaced by their three-dimensional counterparts. The representation theory of $\SO(3)$ is more complicated than that of $\SO(2)$, but it is similarly possible to design roto-translationally equivariant convolutions in three-dimensional~\cite{weiler_3d_2018}. One potential application is mentioned in Remark~\ref{remark:vector_valued}: in diffusion tensor MRI, the domain is three-dimensional, with the additional challenge that the image that is to be recovered is a tensor field rather than a scalar field.

In the experiments that we demonstrated in this work, we focused on a single type of learned reconstruction operator, the learned proximal gradient method. In fact, the framework that we describe is not limited to this form of reconstruction algorithm. As an example of another type of learned reconstruction operator, consider the learned primal-dual method of~\cite{adler_learned_2018}. A small corollary to Proposition~\ref{prop:invariant_regulariser_implies_equivariant_prox} is that, when $J$ is invariant and the Fenchel conjugate $J^*$ is well-defined, $\prox_{J^*}$ will be equivariant in the same way that $\prox_{J}$ is. As a result, assuming reasonable invariance properties of a data discrepancy term, a learned primal-dual method can be considered where both the primal and dual proximal operators are modeled as appropriate equivariant neural networks.

\section*{Acknowledgements}
Data used in the preparation of this article were obtained from the NYU fastMRI Initiative database (\texttt{fastmri.med.nyu.edu}) \cite{knoll_fastmri_2020, zbontar_fastmri_2019}. As such, NYU fastMRI investigators provided data but did not participate in analysis or writing of this report. A listing of NYU fastMRI investigators, subject to updates, can be found at \texttt{fastmri.med.nyu.edu}. The primary goal of fastMRI is to test whether machine learning can aid in the reconstruction of medical images.

The authors acknowledge the National Cancer Institute and the Foundation for the National Institutes of Health, and their critical role in the creation of the free publicly available LIDC/IDRI Database used in this study \cite{armato_iii_lung_2011, armato_iii_data_2015}.

MJE acknowledges support from the EPSRC grants EP/S026045/1 and \linebreak EP/T026693/1, the Faraday Institution via EP/T007745/1, and the Leverhulme Trust fellowship ECF-2019-478. 

CE and CBS acknowledge support from the Wellcome Innovator Award RG98755.

CBS acknowledges support from the Leverhulme Trust project on ‘Breaking the non-convexity barrier’, the Philip Leverhulme Prize, the EPSRC grants EP/S026045/1 and EP/T003553/1, the EPSRC Centre Nr. EP/N014588/1, European Union Horizon 2020 research and innovation programmes under the Marie Sk\l{}odowska-Curie grant agreement No. 777826 NoMADS and No. 691070 CHiPS, the Cantab Capital Institute for the Mathematics of Information and the Alan Turing Institute. 

FS acknowledges support from the Cantab Capital Institute for the Mathematics of Information. 

EC and BO thank the SPIRIT project (No. 231632) under the Research Council of Norway FRIPRO funding scheme. 

\bibliographystyle{unsrt}
\bibliography{main.blg}

\end{document}